\documentclass[sigconf]{aamas}  
\AtBeginDocument{%
  \providecommand\BibTeX{{%
    \normalfont B\kern-0.5em{\scshape i\kern-0.25em b}\kern-0.8em\TeX}}}
    
\renewcommand\footnotetextcopyrightpermission[1]{} 

\usepackage{booktabs}

\usepackage{flushend}
\setcopyright{ifaamas}  
\copyrightyear{2020} 
\acmYear{2020} 
\acmDOI{} 
\acmPrice{} 
\acmISBN{} 
\acmConference[AAMAS'20]{Proc.\@ of the 19th International Conference on Autonomous Agents and Multiagent Systems (AAMAS 2020)}{May 9--13, 2020}{Auckland, New Zealand}{B.~An, N.~Yorke-Smith, A.~El~Fallah~Seghrouchni, G.~Sukthankar (eds.)}  

\usepackage{amssymb,amsmath,stmaryrd,graphics,wrapfig,url}
\renewcommand{\phi}{\varphi}
\newtheorem{claim}{Claim}

\renewcommand{\phi}{\varphi}
\newcommand{\K}{{\sf K}}

\newcommand{\cK}{{\sf \overline{K}}}

\newcommand{\B}{{\sf B}}

\newenvironment{proof-of-claim}{\noindent{\em Proof of Claim.}}{\hfill $\boxtimes\hspace{2mm}$\linebreak}

\title{Duty to Warn in Strategic Games}

\author{Pavel Naumov}
\orcid{0000-0003-1687-045X}
\affiliation{%
 \institution{Tulane University}
 \city{New Orleans} 
 \state{Louisiana} 
}
\email{pgn2@cornell.edu}

\author{Jia Tao}
\affiliation{%
 \institution{Lafayette College}
 \city{Easton} 
 \state{Pennsylvania} 
}
\email{taoj@lafayette.edu}

\begin{document}

\begin{abstract}
The paper investigates the second-order blameworthiness or duty to warn modality ``one coalition knew how another coalition could have prevented an outcome''. The main technical result is a sound and complete logical system that describes the interplay between the distributed knowledge and the duty to warn modalities.
\end{abstract}

\keywords{epistemic logic; blameworthiness; responsibility completeness}  

\maketitle

\section{Introduction}

On October 27, 1969, Prosenjit Poddar, an Indian graduate student from the University of California, Berkeley, came to the parents' house of Tatiana Tarasoff, an undergraduate student who recently immigrated from Russia. After a brief conversation, he pulled out a gun and unloaded it into her torso, then stabbed her eight times with a 13-inch butcher knife, walked into the house and called the police. Tarasoff was pronounced dead on arrival at the hospital~\cite{a17timeline}. 

In this paper we study the notion of blameworthiness. This notion is usually defined through the principle of alternative possibilities: an agent (or a coalition of agents) is blamable for $\phi$ if $\phi$ is true and the agent had a strategy to prevent it~\cite{f69tjop,w17}. This definition is also referred to as the  counterfactual definition of blameworthiness~\cite{c15cop}. In our case, Poddar is blamable for the death of Tatiana because he could have taken actions (to refrain from shooting and stabbing her) that would have prevented her death. He was found guilty of second-degree murder and sentenced to five years~\cite{a17timeline}.
The principle of alternative possibilities, sometimes referred to as ``counterfactual possibility''~\cite{c15cop}, is also used to define causality~\cite{lewis13,h16,bs18aaai}. A sound and complete axiomatization of modality ``statement $\phi$ is true and coalition $C$ had a strategy to prevent $\phi$'' is proposed in~\cite{nt19aaai}. In related works, Xu~\cite{x98jpl} and  Broersen, Herzig, and Troquard~\cite{bht09jancl} axiomatized modality ``took actions that unavoidably resulted in $\phi$'' in the cases of single agents and coalitions respectively.

According to the principle of alternative possibilities, Poddar is not the only one who is blamable for Tatiana's death. Indeed, Tatiana's parents could have asked for a temporary police protection, hired a private bodyguard, or taken Tatiana on a long vacation outside of California. Each of these actions is likely to prevent Tatiana's death. Thus, by applying the principle of alternative possibilities directly, we have to conclude that her parents should be blamed for Tatiana's death. However, the police is unlikely to provide life-time protection; the parents' resources can only be used to hire a bodyguard for a limited period time; and any vacation will have to end. These measures would only work if they knew an approximate time of a likely attack on their daughter. Without this crucial information, they had a strategy to prevent her death, but they did not know what this strategy was. If an agent has a strategy to achieve a certain outcome, knows that it has a strategy, and knows what this strategy is, then we say that the agent has a {\em know-how strategy}. Axiomatic systems for know-how strategies have been studied before~\cite{aa16jlc,fhlw17ijcai,nt17aamas,nt18ai,nt18aaai,nt19ai}. In a setting with imperfect information, it is natural to modify the principle of alternative possibilities to require an agent or a coalition to have a know-how strategy to prevent. In our case, parents had many different strategies that included taking vacations in different months. They did not know that a vacation in October would have prevented Tatiana's death. Thus, they cannot be blamed for her death according to the modified version of the principle of alternative possibilities. We write this as 
$
\neg\B_{\mbox{\scriptsize parents}}(\mbox{``Tatiana is killed''}).
$ 

Although Tatiana's parents did not know how to prevent her death, Dr. Lawrence Moore did. He was a psychiatrist who treated Poddar at the University of California mental clinic. Poddar told Moore how he met Tatiana at the University international student house, how they started to date and how depressed Poddar became when Tatiana lost romantic interest in him. Less than two months before the tragedy, Poddar shared with the doctor his intention to buy a gun and to murder Tatiana. Dr. Moore reported this information to the University campus police. Since the University knew that Poddar was at the peak of his depression, they could estimate the possible timing of the attack. Thus, the University knew what actions the parents could take to prevent the tragedy. In general, if a coalition $C$ knows how a coalition $D$ can achieve a certain outcome, then coalition $D$ has a {\em second-order know-how} strategy to achieve the outcome. This class of strategies and a complete logical system that describes its properties were proposed in~\cite{nt18aamas}. We write $\B_C^D\phi$ if {\em $\phi$ is true and coalition $C$ knew how coalition $D$ could have prevented $\phi$}.
In our case,
$
\B_{\mbox{\scriptsize university}}^{\mbox{\scriptsize parents}}(\mbox{``Tatiana is killed''}).
$

After Tatiana's death, her parents sued the University. In 1976 the California Supreme Court ruled that ``When a therapist determines, or pursuant to the standards of his profession should determine, that his patient presents a serious danger of violence to another, he incurs an obligation to use reasonable care to protect the intended victim against such danger. The discharge of this duty may require the therapist to take one or more of various steps, depending upon the nature of the case. Thus it may call for him to warn the intended victim or others likely to apprise the victim of the danger, to notify the police, or to take whatever other steps are reasonably necessary under the circumstances.''~\cite{t76opinion}.  
In other words, the California Supreme Court ruled that 
in this case the duty to warn is not only a moral obligation but a legal one as well. 
In this paper we propose a sound and complete logical system that describes the interplay between the distributed knowledge modality $\K_C$ and the second-order blameworthiness or {\em duty to warn} modality $\B^D_C$. The (first-order) blameworthiness modality $\B_C\phi$ mentioned earlier could be viewed as an abbreviation for $\B_C^C\phi$. For example, 
$
\B_{\mbox{\scriptsize Poddar}}(\mbox{``Tatiana is killed''})
$
because Poddar knew how he himself could prevent Tatiana's death.
 
The paper is organized as follows. In the next section we introduce and discuss the formal syntax and semantics of our logical system. In Section~\ref{axioms section} we list axioms and compare them to those in the related logical systems. Section~\ref{examples section} gives examples of formal proofs in our system.  Section~\ref{soundness section} and Section~\ref{completeness section} contain the proofs of the soundness and the completeness, respectively. Section~\ref{conclusion section} concludes.

\section{Syntax and Semantics}\label{syntax and semantics section}

In this section we introduce the formal syntax and semantics of our logical system. We assume a fixed set of propositional variables and a fixed set of agents $\mathcal{A}$. By a coalition we mean any subset of $\mathcal{A}$.
The language $\Phi$ of our logical system is defined by  grammar:
$$
\phi := p\;|\;\neg\phi\;|\;\phi\to\phi\;|\;\K_C\phi\;|\;\B^D_C\phi,
$$
where $C$ and $D$ are arbitrary coalitions. Boolean connectives $\bot$, $\wedge$, and $\vee$ are defined through $\neg$ and $\to$ in the usual way. By $\cK_C\phi$ we denote the formula $\neg\K_C\neg\phi$ and by $X^Y$ the set of all functions from set $Y$ to set $X$.

\begin{definition}\label{game definition}
A game is a tuple $\left(I, \{\sim_a\}_{a\in\mathcal{A}},\Delta,\Omega,P,\pi\right)$, where 
\begin{enumerate}
    \item $I$ is a set of ``initial states'',
    \item $\sim_a$ is an ``indistinguishability'' equivalence relation on the set of initial states $I$, for each agent $a\in\mathcal{A}$,
    \item $\Delta$ is a set of ``actions'',
    \item $\Omega$ is a set of ``outcomes'',
    \item a set of ``plays'' $P$ is an arbitrary set of tuples $(\alpha,\delta,\omega)$ such that $\alpha\in I$, $\delta\in\Delta^\mathcal{A}$, and $\omega\in\Omega$. Furthermore, we assume that for each initial state $\alpha\in I$ and each function $\delta\in\Delta^\mathcal{A}$, there is at least one outcome $\omega\in \Omega$ such that $(\alpha,\delta,\omega)\in P$,
    \item $\pi(p)\subseteq P$ for each propositional variable $p$.
\end{enumerate}
\end{definition} 
By a complete (action) profile we mean any function $\delta\in\Delta^\mathcal{A}$ that maps agents in $\mathcal{A}$ into actions in $\Delta$. By an (action) profile of a coalition $C$ we mean any function from set $\Delta^C$.

\begin{figure}[ht]
\begin{center}
\vspace{-2mm}
\scalebox{0.4}{\includegraphics{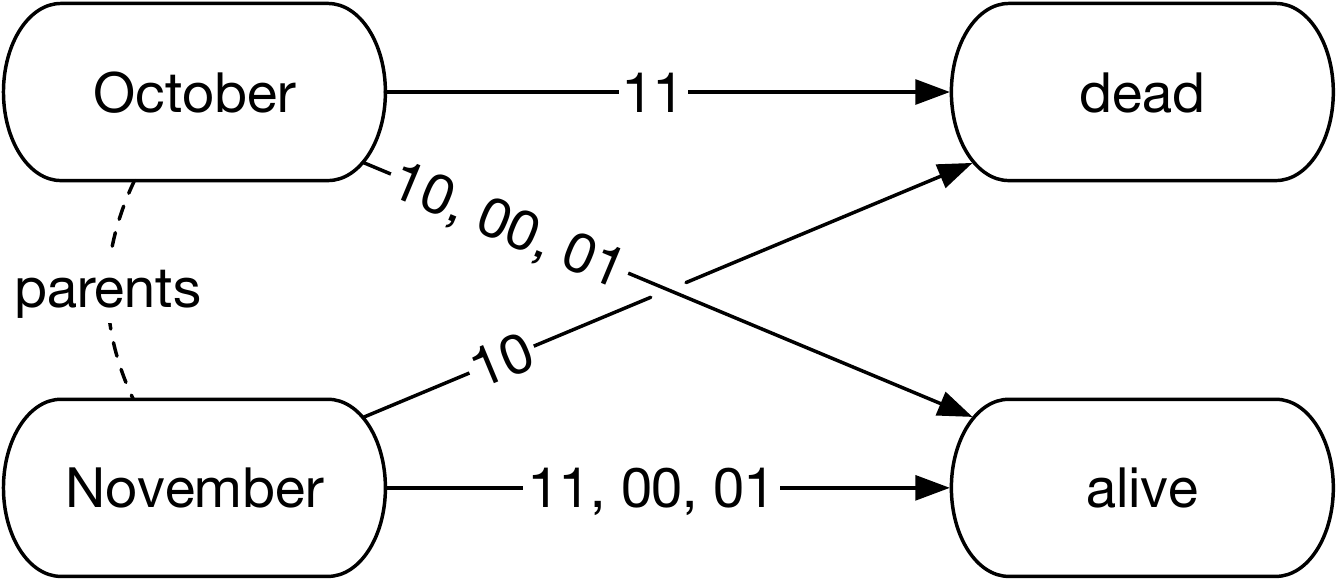}}
\vspace{0mm}
\caption{Poddar's actions: not attack (0) or attack (1). Parents' actions: take vacation in October (0) or November (1).}\label{example figure}
\vspace{0mm}
\end{center}
\vspace{-2mm}
\end{figure}

Figure~\ref{example figure} depicts a diagram of the game for the Tarasoff case. It shows two possible initial states:  October and November that represent two possible months with the peak of Poddar's depression. The actual initial state was October, which was known to the University, but not to Tatiana's parents. In other words, the University could distinguish these two states, but the parents could not. We show the indistinguishability relation by dashed lines. At the peak of his depression,  agent Poddar might decide not to attack Tatiana (action 0) or to attack her (action 1). Parents, whom we represent by a single agent for the sake of simplicity, might decide to take vacation in October (action 0) or November (action 1). Thus, in our example, $\Delta=\{0,1\}$. Set $\Omega$ consists of outcomes  $dead$ and $alive$. Recall that a complete action profile is a function from agents into actions. Since in our case there are only two agents (Poddar and parents), we write action profiles as $xy$ where $x\in \{0,1\}$ is an action of Poddar and $y\in\{0,1\}$ is an action of the parents. The plays of the game are all possible valid combinations of an initial state, a complete action profile, and an outcome. The plays are represented in the diagram by directed edges. For example, the directed edge from initial state October to outcome $dead$ is labeled with action profile $11$. This means that $(\mbox{October},11,dead)\in P$. In other words, if the peak of depression is in October, Poddar decides to attack (1), and the parents take vacation in November (1), then Tatiana is dead. Multiple labels on the same edge of the diagram represent multiple plays with the same initial state and the same outcome.  

Function $\pi$ specifies the meaning of propositional variables. Namely, $\pi(p)$ is the set of all plays for which proposition $p$ is true.

Next is the core definition of this paper. Its item 5 formally defines the semantics of modality $\B_C^D$.
Traditionally, in modal logic the satisfiability $\Vdash$ is defined as a relation between a state and a formula. This approach is problematic in the case of the blameworthiness modality because this modality refers to two different states: $\B^D_C\phi$ if statement $\phi$ is true in {\em the current} state and coalition $C$ knew how coalition $D$ could have prevented $\phi$ in {\em the previous} state. In other words, the meaning of formula $\B^D_C\phi$ depends not only on the current state, but on the previous one as well. We resolve this issue by defining the satisfiability as a relation between a {\em play} and a formula, where a play is a triple consisting of the previous state $\alpha$, the complete action profile $\delta$, and an outcome (state) $\omega$. We distinguish initial states from outcomes to make the presentation more elegant. Otherwise, this distinction is not significant.

We write $\omega\sim_C\omega'$ if $\omega\sim_a\omega'$ for each agent $a\in C$. We also write $f=_X g$ if $f(x)=g(x)$ for each element $x\in X$.

\begin{definition}\label{sat} 
For any game $\left(I, \{\sim_a\}_{a\in\mathcal{A}},\Delta,\Omega,P,\pi\right)$, any formula $\phi\in\Phi$, and any play $(\alpha,\delta,\omega)\in P$, 
the satisfiability relation $(\alpha,\delta,\omega)\Vdash\phi$ is defined recursively as follows:
\begin{enumerate}
    \item $(\alpha,\delta,\omega)\Vdash p$ if $(\alpha,\delta,\omega)\in \pi(p)$,
    \item $(\alpha,\delta,\omega)\Vdash \neg\phi$ if $(\alpha,\delta,\omega)\nVdash \phi$,
    \item $(\alpha,\delta,\omega)\Vdash\phi\to\psi$ if $(\alpha,\delta,\omega)\nVdash\phi$ or $(\alpha,\delta,\omega)\Vdash\psi$,
    \item $(\alpha,\delta,\omega)\Vdash\K_C\phi$ if $(\alpha',\delta',\omega')\Vdash\phi$ for each $(\alpha',\delta',\omega')\in P$ such that $\alpha\sim_C\alpha'$, 
    \item $(\alpha,\delta,\omega)\Vdash\B_C^D\phi$ if $(\alpha,\delta,\omega)\Vdash\phi$ and there is a profile $s\in \Delta^D$ such that for each play $(\alpha',\delta',\omega')\in P$, if $\alpha\sim_C\alpha'$ and $s=_D\delta'$, then $(\alpha',\delta',\omega')\nVdash\phi$.
\end{enumerate}
\end{definition}


Going back to our running example, 
$$
(\mbox{October},11,dead)\Vdash\B_{\mbox{\scriptsize university}}^{\mbox{\scriptsize parents}}(\mbox{``Tatiana is killed''})
$$
because
$
(\mbox{October},11,dead)\Vdash\mbox{``Tatiana is killed''}
$
and
$$
(\alpha',\delta',\omega')\nVdash(\mbox{``Tatiana is killed''})
$$
for each play $(\alpha',\delta',\omega')\in P$ such that $\alpha'\sim_{\mbox{\scriptsize university}}\mbox{October}$ and $\delta'(\mbox{parents})=0$.


Because the satisfiability  is defined as a relation between plays and formulae, one can potentially talk about two forms of knowledge {\em about a play} in our system: {\em a priori} knowledge in the initial state and {\em a posteriori} knowledge in the outcome. The knowledge captured by the  modality $\K$ as well as the knowledge implicitly referred to by the modality $\B$, see item (5) of Definition~\ref{sat}, is {\em a priori} knowledge about a play. In order to define a posteriori knowledge in our setting, one would need to add an indistinguishability relation on outcomes to Definition~\ref{game definition}. We do not consider a posteriori knowledge because one should not be blamed for something that the person only knows how to prevent {\em post-factum}.


Since we define the second-order blameworthiness using distributed knowledge, if a coalition $C$ is blamable for not warning coalition $D$, then any superset $C'\supseteq C$ could be blamed for not warning $D$. One might argue that the definition of blameworthiness modality $\B^D_C$ should include a minimality condition on the coalition $C$. We do not include this condition in item (5) of Definition~\ref{sat}, because there are several different ways to phrase the minimality, all of which could be expressed through our basic modality $\B^D_C$. 

First of all, we can say that $C$ is the minimal coalition among those coalitions that knew how $D$ could have prevented $\phi$. Let us denote this modality by $[1]^D_C\phi$. It can be expressed through $\B^D_C$ as:
$$
[1]^D_C\phi\equiv\B^D_C\phi\wedge \neg\bigvee_{E\subsetneq C}\B^D_E\phi.
$$

Second, we can say that $C$ is the minimal coalition that knew how {\em somebody} could have prevented $\phi$:
$$
[2]^D_C\phi\equiv\B^D_C\phi\wedge \neg\bigvee_{E\subsetneq C}\bigvee_{F\subseteq\mathcal{A}}\B^F_E\phi.
$$

Third, we can say that $C$ is the minimal coalition that knew how {\em the smallest} coalition $D$ could have prevented $\phi$: 
$$
[3]^D_C\phi\equiv\B^D_C\phi\wedge \neg\bigvee_{E\subseteq \mathcal{A}}\bigvee_{F\subsetneq D}\B^F_E\phi \wedge \neg\bigvee_{E\subsetneq C}\B^D_E\phi.
$$

Finally, we can say that $C$ is the minimal coalition that knew how {\em some smallest} coalition could have prevented $\phi$: 
$$
[4]_C\phi\equiv \bigvee_{D\subseteq\mathcal{A}}\left(\B^D_C\phi\wedge \neg\bigvee_{E\subseteq \mathcal{A}}\bigvee_{F\subsetneq D}\B^F_E\phi \wedge \neg\bigvee_{E\subsetneq C}\B^D_E\phi\right).
$$
The choice of the minimality condition depends on the specific situation. Instead of making a choice between several possible alternatives, in this paper we study the basic blameworthiness modality without a minimality condition through which modalities  $[1]^D_C\phi$, $[2]^D_C\phi$, $[3]^D_C\phi$, $[4]_C\phi$, and possibly others could be defined.

\section{Axioms}\label{axioms section}

In addition to the propositional tautologies in  language $\Phi$, our logical system contains the following axioms:

\begin{enumerate}
    \item Truth: $\K_C\phi\to\phi$ and $\B^D_C\phi\to\phi$,
    \item Distributivity: $\K_C(\phi\to\psi)\to(\K_C\phi\to \K_C\psi)$,
    \item Negative Introspection: $\neg\K_C\phi\to\K_C\neg\K_C\phi$,
    \item Monotonicity: $\K_C\phi\to\K_E\phi$ and $\B^D_C\phi\to\B^F_E\phi$,\\ where $C\subseteq E$ and $D\subseteq F$,
    \item None to Act: $\neg\B^\varnothing_C\phi$,
    \item Joint Responsibility: if $D\cap F=\varnothing$, then\\ $\cK_C\B^D_C\phi\wedge\cK_E\B_E^F\psi\to (\phi\vee\psi\to\B_{C\cup E}^{D\cup F}(\phi\vee\psi))$,
 
    \item Strict Conditional:
    $\K_C(\phi\to\psi)\to(\B^D_C\psi\to(\phi\to \B^D_C\phi))$,
    \item Introspection of Blameworthiness: $\B^D_C\phi\to\K_C(\phi\to\B^D_C\phi)$.
\end{enumerate}
The Truth, the Distributivity, the Negative Introspection, and the Monotonicity axioms for modality $\K$ are the standard axioms from the epistemic logic S5 for distributed knowledge~\cite{fhmv95}. The Truth axiom for modality $\B$ states that a coalition can only be blamed for something that has actually happened. The Monotonicity axiom for modality $\B$ captures the fact that both distributed knowledge and coalition power are monotonic.

The None to Act axiom is true because the empty coalition has only one action profile. Thus, if the empty coalition can prevent $\phi$, then $\phi$ would have to be false on the current play.  This axiom is similar to the None to Blame axiom  $\neg\B_\varnothing\phi$ in~\cite{nt19aaai}.

%
The Joint Responsibility axiom shows how the blame of two separate coalitions can be combined into the blame of their union. 
This axiom is closely related to Marc Pauly~\cite{p02} Cooperation axiom, which is also used in coalitional modal logics of know-how~\cite{aa16jlc,nt17aamas,nt18ai,nt18aaai} and second-order know-how~\cite{nt18aamas}. We formally prove the soundness of this axiom in Lemma~\ref{joint responsibility soundness}.

Strict conditional  $\K_C(\phi\to\psi)$ states that formula $\phi$ is known to $C$ to imply $\psi$. By contraposition, coalition $C$ knows that if $\psi$ is prevented, then $\phi$ is also prevented. The Strict Conditional axiom states that if $C$ could be second-order blamed for $\psi$, then it should also be second-order blamed for  $\phi$ as long as $\phi$ is true. A similar axiom is present in~\cite{nt19aaai}.

Finally, the Introspection of Blameworthiness axiom says that if coalition $C$ is second-order blamed for $\phi$, then $C$ knows that it is second-order blamed for $\phi$ as long as $\phi$ is true. A similar Strategic Introspection axiom for second-order know-how modality is present in~\cite{nt18aamas}.


We write $\vdash\phi$ if formula $\phi$ is provable from the axioms of our system using the Modus Ponens and
the Necessitation inference rules:
$$
\dfrac{\phi,\phi\to\psi}{\psi},
\hspace{20mm}
\dfrac{\phi}{\K_C\phi}.
$$
We write $X\vdash\phi$ if formula $\phi$ is provable from the theorems of our logical system and an additional set of axioms $X$ using only the Modus Ponens inference rule.

\begin{lemma}\label{super distributivity}
If $\phi_1,\dots,\phi_n\!\vdash\!\psi$, then $\K_C\phi_1,\dots,\K_C\phi_n\!\vdash\!\K_C\psi$.
\end{lemma}
\begin{proof}
By the deduction lemma applied $n$ times, assumption $\phi_1,\dots,\phi_n\vdash\psi$ implies that
$
\vdash\phi_1\to(\phi_2\to\dots(\phi_n\to\psi)\dots).
$
Thus, by the Necessitation inference rule,
$$
\vdash\K_C(\phi_1\to(\phi_2\to\dots(\phi_n\to\psi)\dots)).
$$
Hence, by the Distributivity axiom and the Modus Ponens rule,
$$
\vdash\K_C\phi_1\to\K_C(\phi_2\to\dots(\phi_n\to\psi)\dots).
$$
Then, again by the Modus Ponens rule,
$$
\K_C\phi_1\vdash\K_C(\phi_2\to\dots(\phi_n\to\psi)\dots).
$$
Therefore, $\K_C\phi_1,\dots,\K_C\phi_n\vdash\K_C\psi$ by applying the previous steps $(n-1)$ more times.
\end{proof}

The next lemma capture a well-known property of S5 modality. Its proof could be found, for example, in~\cite{nt18aamas}.

\begin{lemma}[Positive Introspection]\label{positive introspection lemma}
$\vdash \K_C\phi\to\K_C\K_C\phi$. 
\end{lemma}

\section{Examples of Derivations}\label{examples section}

The soundness of our logical system is established in the next section. Here we prove several lemmas about our formal system that will be used later in the proof of the completeness. 

\begin{lemma}\label{alt fairness lemma}
$\vdash\cK_C\B^D_C\phi\to(\phi\to\B^D_C\phi)$.
\end{lemma}
\begin{proof}
Note that $\vdash \B_C^D\phi\to\K_C(\phi\to\B_C^D\phi)$ by the Introspection of Blameworthiness axiom. Thus, $\vdash \neg\K_C(\phi\to\B_C^D\phi)\to \neg\B_C^D\phi$, by the law of contrapositive. Then, $\vdash \K_C(\neg\K_C(\phi\to\B_C^D\phi)\to \neg\B_C^D\phi)$ by the Necessitation inference rule. Hence, by the Distributivity axiom and the Modus Ponens inference rule,
$$\vdash \K_C\neg\K_C(\phi\to\B_C^D\phi)\to \K_C\neg\B_C^D\phi.$$ 
At the same time, by the Negative Introspection axiom:
$$
\vdash \neg\K_C(\phi\to\B_C^D\phi)\to\K_C\neg\K_C(\phi\to\B_C^D\phi).
$$
Then, by the laws of propositional reasoning,
$$\vdash \neg\K_C(\phi\to\B_C^D\phi)\to \K_C\neg\B_C^D\phi.$$
Thus, by the law of contrapositive,
$$\vdash \neg\K_C\neg\B_C^D\phi\to \K_C(\phi\to\B_C^D\phi).$$
Since $\K_C(\phi\to\B_C^D\phi)\to(\phi\to\B_C^D\phi)$ is an instance of the Truth axiom, by propositional reasoning,
$$\vdash \neg\K_C\neg\B_C^D\phi\to (\phi\to\B_C^D\phi).$$
Therefore, $\vdash \cK_C\B_C^D\phi\to (\phi\to\B_C^D\phi)$ by the definition of $\cK_C$.
\end{proof}

\begin{lemma}\label{alt cause lemma}
If $\vdash \phi\leftrightarrow \psi$, then $\vdash \B^D_C\phi\to\B^D_C\psi$.
\end{lemma}
\begin{proof}
By the Strict Conditional axiom,
$$
\vdash \K_C(\psi\to\phi)\to(\B_C^D\phi\to(\psi\to \B_C^D\psi)).
$$
Assumption  $\vdash \phi\leftrightarrow \psi$ implies $\vdash \psi\to \phi$ by the laws of propositional reasoning. Hence, $\vdash \K_C(\psi\to \phi)$ by the Necessitation inference rule. Thus, by the Modus Ponens rule,
$
\vdash \B_C^D\phi\to(\psi\to \B_C^D\psi).
$
Then, by the laws of propositional reasoning,
\begin{equation}\label{sofia}
\vdash (\B_C^D\phi\to\psi)\to (\B_C^D\phi\to \B_C^D\psi).
\end{equation}
Observe that $\vdash \B_C^D\phi\to\phi$ by the Truth axiom. Also, $\vdash \phi\leftrightarrow \psi$ by the assumption of the lemma. Then, by the laws of propositional reasoning, $\vdash \B_C^D\phi\to\psi$. Therefore,
$
\vdash \B_C^D\phi\to \B_C^D\psi
$
by the Modus Ponens inference rule from statement~(\ref{sofia}).
\end{proof}

\begin{lemma}\label{add cK lemma}
$\phi\vdash \cK_C\phi$.
\end{lemma}
\begin{proof}
By the Truth axioms, $\vdash\K_C\neg\phi\to\neg\phi$. Hence, by the law of contrapositive, $\vdash\phi\to \neg\K_C\neg\phi$. Thus, $\vdash\phi\to \cK_C\phi$ by the definition of the modality $\cK_C$. Therefore, $\phi\vdash \cK_C\phi$ by the Modus Ponens inference rule.
\end{proof}

The next lemma generalizes the Joint Responsibility axiom from two coalitions to multiple coalitions. 

\begin{lemma}\label{super joint responsibility lemma}
For any integer $n\ge 0$,
$$
\{\cK_{E_i}\B^{F_i}_{E_i}\chi_i\}_{i=1}^n,\chi_1\vee\dots\vee\chi_n
\vdash \B^{F_1\cup\dots\cup F_n}_{E_1\cup\dots\cup E_n}(\chi_1\vee \dots\vee\chi_n),
$$
where sets $F_1,\dots,F_n$ are pairwise disjoint. 
\end{lemma}
\begin{proof}
We prove the lemma by induction on $n$. If $n=0$, then disjunction $\chi_1\vee\dots\vee \chi_n$ is Boolean constant false $\bot$. Hence, the statement of the lemma, $\bot\vdash\B_\varnothing^\varnothing\bot$, is provable in the propositional logic.

Next, assume that $n=1$. Then, from Lemma~\ref{alt fairness lemma} using Modus Ponens rule twice, we get
$\cK_{E_1}\B_{E_1}^{F_1}\chi_1,\chi_1\vdash\B_{E_1}^{F_1}\chi_1$.

Assume now that $n\ge 2$. By the Joint Responsibility axiom and the Modus Ponens inference rule,
\begin{eqnarray*}
&&\hspace{-8mm}\cK_{E_1\cup \dots \cup E_{n-1}}\B_{E_1\cup \dots \cup E_{n-1}}^{F_1\cup \dots \cup F_{n-1}}(\chi_1\vee\dots\vee\chi_{n-1}),
\cK_{E_n}\B_{E_n}^{F_n}\chi_n,\\
&&\hspace{-8mm}\chi_1\vee\dots\vee\chi_{n-1}\vee\chi_n\vdash \B_{E_1\cup \dots \cup E_{n-1}\cup E_{n}}^{F_1\cup \dots \cup F_{n-1}\cup F_{n}}(\chi_1\vee\dots\vee\chi_{n-1}\vee \chi_n).
\end{eqnarray*}
Hence, by Lemma~\ref{add cK lemma},
\begin{eqnarray*}
&&\hspace{-8mm}\B_{E_1\cup \dots \cup E_{n-1}}^{F_1\cup \dots \cup F_{n-1}}(\chi_1\vee\dots\vee\chi_{n-1}),
\cK_{E_n}\B_{E_n}^{F_n}\chi_n,\chi_1\vee\dots\vee\chi_{n-1}\vee\chi_n\\
&&\vdash \B_{E_1\cup \dots \cup E_{n-1}\cup E_{n}}^{F_1\cup \dots \cup F_{n-1}\cup F_{n}}(\chi_1\vee\dots\vee\chi_{n-1}\vee \chi_n).
\end{eqnarray*}
At the same time, by the induction hypothesis,
$$\{\cK_{E_i}\B_{E_i}^{F_i}\chi_i\}_{i=1}^{n-1},\chi_1\vee\dots\vee\chi_{n-1}
\vdash \B_{E_1\cup \dots \cup E_{n-1}}^{F_1\cup \dots \cup F_{n-1}}(\chi_1\vee \dots\vee\chi_{n-1}).$$
Thus,
\begin{eqnarray*}
&&\hspace{-5mm}\{\cK_{E_i}\B_{E_i}^{F_i}\chi_i\}_{i=1}^n,\chi_1\vee\dots\vee\chi_{n-1},\chi_1\vee\dots\vee\chi_{n-1}\vee\chi_n\\
&&\vdash \B_{E_1\cup \dots \cup E_{n-1}\cup E_{n}}^{F_1\cup \dots \cup F_{n-1}\cup F_{n}}(\chi_1\vee \dots\vee\chi_{n-1}\vee\chi_n).
\end{eqnarray*}
Note that $\chi_1\vee\dots\vee\chi_{n-1}\vdash\chi_1\vee\dots\vee\chi_{n-1}\vee\chi_n$ is provable in the propositional logic. Thus,
\begin{eqnarray}
&&\hspace{-10mm}\{\cK_{E_i}\B_{E_i}^{F_i}\chi_i\}_{i=1}^n,\chi_1\vee\dots\vee\chi_{n-1}\nonumber\\
&&\hspace{0mm} \vdash \B_{E_1\cup \dots \cup E_{n-1}\cup E_{n}}^{F_1\cup \dots \cup F_{n-1}\cup F_{n}}(\chi_1\vee \dots\vee\chi_{n-1}\vee\chi_n).\label{part 1}
\end{eqnarray}
Similarly, by the Joint Responsibility axiom and the Modus Ponens inference rule,
\begin{eqnarray*}
&&\hspace{-8mm}\cK_{E_1}\B_{E_1}^{F_1}\chi_1,\cK_{E_2\cup \dots \cup E_n}\B_{E_2\cup \dots \cup E_n}^{F_2\cup \dots \cup F_n}(\chi_2\vee\dots\vee\chi_n),\\
&&\hspace{-8mm}\chi_1\vee(\chi_2\vee\dots\vee\chi_n)\vdash \B_{E_1\cup \dots \cup E_{n-1}\cup E_n}^{F_1\cup \dots \cup F_{n-1}\cup F_n}(\chi_1\vee(\chi_2\vee\dots\vee \chi_n)).
\end{eqnarray*}
Because formula 
$\chi_1\vee(\chi_2\vee\dots\vee \chi_n)\leftrightarrow \chi_1\vee\chi_2\vee\dots\vee \chi_n$ is provable in the propositional logic, by Lemma~\ref{alt cause lemma},
\begin{eqnarray*}
&&\hspace{-8mm}\cK_{E_1}\B_{E_1}^{F_1}\chi_1,\cK_{E_2\cup \dots \cup E_n}\B_{E_2\cup \dots \cup E_n}^{F_2\cup \dots \cup F_n}(\chi_2\vee\dots\vee\chi_n),\\
&&\hspace{-8mm}\chi_1\vee\chi_2\vee\dots\vee\chi_n\vdash \B_{E_1\cup \dots \cup E_{n-1}\cup E_n}^{F_1\cup \dots \cup F_{n-1}\cup F_n}(\chi_1\vee\chi_2\vee\dots\vee \chi_n).
\end{eqnarray*}
Hence, by Lemma~\ref{add cK lemma},
\begin{eqnarray*}
&&\hspace{-7mm}\cK_{E_1}\B_{E_1}^{F_1}\chi_1,\B_{E_2\cup \dots \cup E_n}^{F_2\cup \dots \cup F_n}(\chi_2\vee\dots\vee\chi_n),\chi_1\vee\chi_2\vee\dots\vee\chi_n\\
&&\vdash \B_{E_1\cup \dots \cup E_{n-1}\cup E_n}^{F_1\cup \dots \cup F_{n-1}\cup F_n}(\chi_1\vee\chi_2\vee\dots\vee \chi_n).
\end{eqnarray*}
At the same time, by the induction hypothesis,
$$
\{\cK_{E_i}\B_{E_i}^{F_i}\chi_i\}_{i=2}^n,\chi_2\vee\dots\vee\chi_n
\vdash \B_{E_2\cup\dots\cup E_n}^{F_2\cup\dots\cup F_n}(\chi_2\vee \dots\vee\chi_n).
$$
Thus,
\begin{eqnarray*}
&&\hspace{-8mm}\{\cK_{E_i}\B_{E_i}^{F_i}\chi_i\}_{i=1}^n,\chi_2\vee\dots\vee\chi_n,\chi_1\vee\chi_2\vee\dots\vee\chi_n\\
&&\vdash \B_{E_1\cup\dots\cup E_{n-1}\cup E_n}^{F_1\cup\dots\cup F_{n-1}\cup F_n}(\chi_1\vee\chi_2\vee\dots\vee\chi_n).
\end{eqnarray*}
Note that $\chi_2\vee\dots\vee\chi_{n}\vdash\chi_1\vee\dots\vee\chi_{n-1}\vee\chi_n$ is provable in the propositional logic. Thus,
\begin{eqnarray}
&&\hspace{-10mm}\{\cK_{E_i}\B_{E_i}^{F_i}\chi_i\}_{i=1}^n,\chi_2\vee\dots\vee\chi_n\nonumber\\
&&\hspace{0mm} \vdash \B_{E_1\cup\dots\cup E_{n-1}\cup E_n}^{F_1\cup\dots\cup F_{n-1}\cup F_n}(\chi_1\vee\chi_2\vee\dots\vee\chi_n).\label{part 2}
\end{eqnarray}
Finally, note that the following statement is provable in the propositional logic for $n\ge 2$,
$$
\vdash\chi_1\vee\dots\vee\chi_n\to(\chi_1\vee\dots\vee\chi_{n-1})\vee 
(\chi_2\vee\dots\vee\chi_n).
$$
Therefore, from statement~(\ref{part 1}) and statement~(\ref{part 2}),
$$
\{\cK_{E_i}\B_{E_i}^{F_i}\chi_i\}_{i=1}^n,\chi_1\vee\dots\vee\chi_n
\vdash \B_{E_1\cup\dots\cup E_n}^{F_1\cup\dots\cup F_n}(\chi_1\vee \dots\vee\chi_n).
$$
by the laws of propositional reasoning.
\end{proof}

Our last example rephrases Lemma~\ref{super joint responsibility lemma} into the form which is used in the proof of the completeness.

\begin{lemma}\label{five plus plus}
For any $n\ge 0$, any sets $E_1,\dots,E_n\subseteq C$, and any pairwise disjoint sets $F_1,\dots,F_n\subseteq D$,
$$
\{\cK_{E_i}\B^{F_i}_{E_i}\chi_i\}_{i=1}^n,\K_C(\phi\to\chi_1\vee\dots\vee\chi_n)\vdash\K_C(\phi\to\B_C^D\phi).
$$
\end{lemma}
\begin{proof} Let $X=\{\cK_{E_i}\B^{F_i}_{E_i}\chi_i\}_{i=1}^n$.
Then, by Lemma~\ref{super joint responsibility lemma},
$$
X,\chi_1\vee\dots\vee\chi_n\vdash \B_{E_1\cup\dots\cup E_n}^{F_1\cup\dots\cup F_n}(\chi_1\vee\dots\vee\chi_n).
$$
Hence, by the Monotonicity axiom, 
$$
X,\chi_1\vee\dots\vee\chi_n\vdash \B_{C}^D(\chi_1\vee\dots\vee\chi_n).
$$
$$
\hspace{-15mm}\mbox{Thus, }\hspace{5mm}X,\phi, \phi\to\chi_1\vee\dots\vee\chi_n\vdash \B_{C}^D(\chi_1\vee\dots\vee\chi_n)
$$
by the Modus Ponens inference rule.
Hence, by the Truth axiom,
$$
X,\phi, \K_C(\phi\to\chi_1\vee\dots\vee\chi_n)\vdash \B_{C}^D(\chi_1\vee\dots\vee\chi_n).
$$

The following formula is an instance of the Strict Conditional axiom $\K_C(\phi\to\chi_1\vee\dots\vee\chi_n)\to(\B_C^D(\chi_1\vee\dots\vee\chi_n)\to(\phi\to\B_C^D\phi))$. Thus, by the Modus Ponens applied twice,
$$
X,\phi, \K_C(\phi\to\chi_1\vee\dots\vee\chi_n)\vdash \phi\to\B_C^D\phi.
$$
Then, 
$
X,\phi, \K_C(\phi\to\chi_1\vee\dots\vee\chi_n)\vdash \B_C^D\phi
$
by the Modus Ponens. \\ Thus,
$
X, \K_C(\phi\to\chi_1\vee\dots\vee\chi_n)\vdash \phi\to\B_C^D\phi
$
by the deduction lemma.
Hence,
$$
\{\K_C\cK_{E_i}\B_{E_i}^{F_i}\chi_i\}_{i=1}^n, \K_C\K_C(\phi\to\chi_1\vee\dots\vee\chi_n)\vdash \K_C(\phi\to\B_C^D\phi)
$$
by Lemma~\ref{super distributivity} and the definition of set $X$.
Then,
$$
\{\K_{E_i}\cK_{E_i}\B_{E_i}^{F_i}\chi_i\}_{i=1}^n, \K_C\K_C(\phi\to\chi_1\vee\dots\vee\chi_n)\vdash \K_C(\phi\to\B_C^D\phi)
$$
by the Monotonicity axiom, the Modus Ponens inference rule, and the assumption $E_1,\dots,E_n\subseteq C$.
Thus,
$$
\{\cK_{E_i}\B_{E_i}^{F_i}\chi_i\}_{i=1}^n, \K_C\K_C(\phi\to\chi_1\vee\dots\vee\chi_n)\vdash \K_C(\phi\to\B_C^D\phi)
$$
by the definition of modality $\cK$, the Negative Introspection axiom, and the Modus Ponens rule.
Therefore, by Lemma~\ref{positive introspection lemma} and the Modus Ponens inference rule, the statement of the lemma is true. 
\end{proof}




% ...


% ...


\section{Soundness}\label{soundness section} 

The soundness of the Truth, the Distributivity, the Negative Introspection, the Monotonicity, and the None to Blame axioms is straightforward. Below we prove the soundness of the  Joint Responsibility, the  Strict Conditional, and the Introspection of Blameworthiness axioms as separate lemmas.


\begin{lemma}\label{joint responsibility soundness}
If $D\cap F=\varnothing$, $(\alpha,\delta,\omega)\Vdash \cK_C\B_C^D\phi$, $(\alpha,\delta,\omega)\Vdash \cK_E\B_E^F\psi$, and $(\alpha,\delta,\omega)\Vdash \phi\vee\psi$, then $(\alpha,\delta,\omega)\Vdash \B_{C\cup E}^{D\cup F}(\phi\vee\psi)$.
 \end{lemma}
\begin{proof}
By Definition~\ref{sat} and the definition of modality $\cK$, assumption $(\alpha,\delta,\omega)\Vdash\cK_C\B^D_C\phi$ implies that there is a play $(\alpha_1,\delta_1,\omega_1)$ such that $\alpha\sim_C\alpha_1$ and $(\alpha_1,\delta_1,\omega_1)\Vdash\B^D_C\phi$. Thus, again by Definition~\ref{sat}, there is an action profile $s_1\in\Delta^D$ such that for each play $(\alpha',\delta',\omega')\in P$, if $\alpha_1\sim_C\alpha'$ and $s_1=_D\delta'$, then $(\alpha',\delta',\omega')\nVdash\phi$. Recall that $\alpha\sim_C\alpha_1$. Thus, for each play $(\alpha',\delta',\omega')\in P$,
\begin{equation}\label{november}
    \alpha\sim_C\alpha'\wedge s_1=_D\delta'\to (\alpha',\delta',\omega')\nVdash\phi.
\end{equation}
Similarly, assumption $(\alpha,\delta,\omega)\Vdash\cK_E\B^F_E\psi$ implies that there is a profile $s_2\in\Delta^F$ such that for each play $(\alpha',\delta',\omega')\in P$,
\begin{equation}\label{december}
    \alpha\sim_E\alpha'\wedge s_2=_F\delta'\to (\alpha',\delta',\omega')\nVdash\psi.
\end{equation}
Let $s\in\Delta^{D\cup F}$ be the action profile:
\begin{equation}\label{january}
s(a)=
\begin{cases}
s_1(a), & \mbox{if } a\in D,\\
s_2(a), & \mbox{if } a\in F.
\end{cases}
\end{equation}
Action profile $s$ is well-defined because $D\cap F=\varnothing$. Statements~(\ref{november}), (\ref{december}), and (\ref{january}) by Definition~\ref{sat} imply that for each play $(\alpha',\delta',\omega')\in P$ if $\alpha\sim_{C\cup E}\alpha'$ and $s=_{D\cup F}\delta'$, then $(\alpha',\delta',\omega')\nVdash\phi\vee\psi$. Recall that $(\alpha,\delta,\omega)\Vdash\phi\vee\psi$. Therefore, $(\alpha,\delta,\omega)\Vdash\B^{D\cup F}_{C\cup E}(\phi\vee\psi)$ by Definition~\ref{sat}.
\end{proof}

\begin{lemma}
If $(\alpha,\delta,\omega)\Vdash \K_C(\phi\to\psi)$, $(\alpha,\delta,\omega)\Vdash \B_C^D\psi$, and $(\alpha,\delta,\omega)\Vdash \phi$, then $(\alpha,\delta,\omega)\Vdash \B_C^D\phi$.
\end{lemma}
\begin{proof}
By Definition~\ref{sat}, assumption $(\alpha,\delta,\omega)\Vdash \K_C(\phi\to\psi)$ implies that for each play $(\alpha',\delta',\omega')\in P$ of the game if $\alpha\sim_C\alpha'$, then
$(\alpha',\delta',\omega')\Vdash\phi\to\psi$. 

By Definition~\ref{sat}, assumption $(\alpha,\delta,\omega)\Vdash \B^D_C\psi$ implies that there is an action profile $s\in \Delta^D$ such that for each play $(\alpha',\delta',\omega')\in P$, if $\alpha\sim_C\alpha'$ and $s=_D\delta'$, then $(\alpha',\delta',\omega')\nVdash\psi$. 

Hence, for each play $(\alpha',\delta',\omega')\in P$, if $\alpha\sim_C\alpha'$ and $s=_D\delta'$, then $(\alpha',\delta',\omega')\nVdash\phi$.
Therefore, $(\alpha,\delta,\omega)\Vdash \B^D_C\phi$ by Definition~\ref{sat} and the assumption $(\alpha,\delta,\omega)\Vdash \phi$ of the lemma.
\end{proof}

\begin{lemma}
If $(\alpha,\delta,\omega)\Vdash \B^D_C\phi$, then $(\alpha,\delta,\omega)\Vdash \K_C(\phi\to\B^D_C\phi)$.
\end{lemma}
\begin{proof} 
By Definition~\ref{sat}, assumption $(\alpha,\delta,\omega)\Vdash \B^D_C\phi$ implies that there is an action profile $s\in \Delta^D$ such that for each play $(\alpha',\delta',\omega')\in P$, if $\alpha\sim_C\alpha'$ and $s=_D\delta'$, then $(\alpha',\delta',\omega')\nVdash\phi$. 

Let $(\alpha',\delta',\omega')\in P$ be a play where $\alpha\sim_C\alpha'$ and $(\alpha',\delta',\omega')\Vdash \phi$. By Definition~\ref{sat}, it suffices to show that $(\alpha',\delta',\omega')\Vdash \B^D_C\phi$. 

Consider any play $(\alpha'',\delta'',\omega'')\in P$ such that $\alpha'\sim_C\alpha''$ and $s=_D\delta''$.  
Then, since $\sim_C$ is an equivalence relation, assumptions $\alpha\sim_C\alpha'$ and $\alpha'\sim_C\alpha''$ imply $\alpha\sim_C\alpha''$. Thus, $(\alpha'',\delta'',\omega'')\nVdash\phi$ by the choice of action profile $s$. Therefore, $(\alpha',\delta',\omega')\Vdash \B^D_C\phi$ by Definition~\ref{sat} and the assumption $(\alpha',\delta',\omega')\Vdash \phi$.
\end{proof}

\section{Completeness}\label{completeness section}

The standard proof of the completeness for individual knowledge modality $\K_a$ defines states as maximal consistent sets~\cite{fhmv95}. Two such sets are indistinguishable to an agent $a$ if these sets have the same $\K_a$-formulae. This construction does {\em not} work for distributed knowledge because if two sets share $\K_a$-formulae and $\K_b$-formulae, they do not necessarily have to share $\K_{a,b}$-formulae. To overcome this issue, we use the Tree of Knowledge construction, similar to the one in~\cite{nt18ai}. An important change to this construction proposed in the current paper is placing elements of a set $\mathcal{B}$ on the edges of the tree. This change is significant for the proof of Lemma~\ref{B child exists lemma}.

Let $\mathcal{B}$ be  an arbitrary set of cardinality larger than that of the set  $\mathcal{A}$.
Next, 
for each maximal consistent set of formulae $X_0$,
we define the canonical game $G(X_0)=\left(I,\{\sim_a\}_{a\in\mathcal{A}},\Delta,\Omega,P,\pi\right)$. 

\begin{definition}\label{canonical outcome}
The set of outcomes $\Omega$ consists of all sequences $X_0,(C_1,b_1),X_1,(C_2,b_2),\dots,(C_n,b_n),X_n$, where $n\ge 0$ and for each $i\ge 1$, $X_i$ is a maximal consistent subset of $\Phi$, (i) $C_i\subseteq\mathcal{A}$,  (ii) $b_i\in\mathcal{B}$, and (iii) $\{\phi\;|\;\K_{C_i}\phi\in X_{i-1}\}\subseteq X_i$. 
\end{definition}

If $x$ is a nonempty sequence $x_1,\dots,x_n$ and $y$ is an element, then by $x::y$ and $hd(x)$ we mean sequence $x_1,\dots,x_n,y$ and element $x_n$ respectively. 

We say that outcomes $w,u\in\Omega$ are {\em adjacent} if there are coalition $C$, element $b\in\mathcal{B}$, and maximal consistent set $X$ such that $w=u::(C,b)::X$. The adjacency relation forms a tree structure on set $\Omega$, see Figure~\ref{tree figure}. 
We call it {\em the Tree of Knowledge}. We say that edge $(w,u)$ is {\em labeled} with each agent in coalition $C$ and is {\em marked} with element $b$. Although vertices of the tree are sequences, it is convenient to think about the maximal consistent set $hd(\omega)$, not a sequence $\omega$, being a vertex of the tree.

\begin{figure}[ht]
\begin{center}
\vspace{0mm}
\scalebox{0.55}{\includegraphics{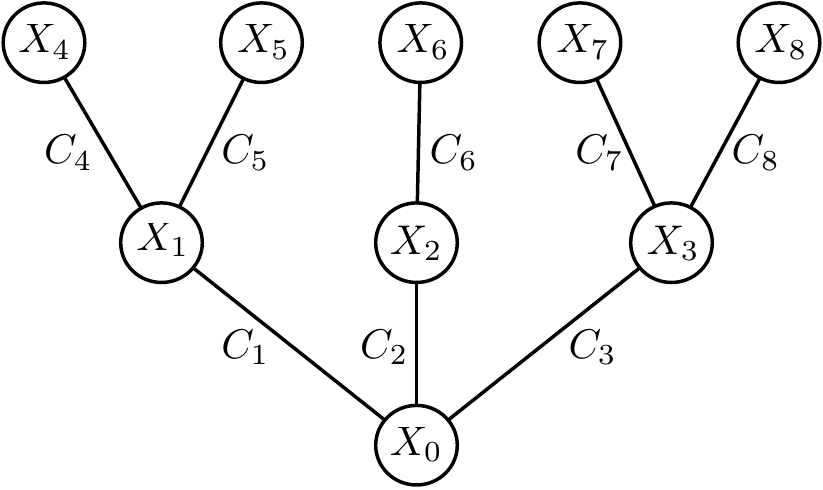}}
\caption{A Fragment of the Tree of Knowledge.}\label{tree figure}
\vspace{-4mm}
\end{center}
\vspace{0mm}
\end{figure}

\begin{definition}
For any outcome $\omega\in\Omega$, let $Tree(\omega)$ be the set of all $\omega'\in\Omega$ such that sequence $\omega$ is a prefix of sequence $\omega'$.
\end{definition}
Note that $Tree(\omega)$ is a subtree of the Tree of Knowledge rooted at vertex $\omega$, see Figure~\ref{tree figure}.

\begin{definition}\label{canonical sim}
For any two outcomes $\omega,\omega'\in\Omega$ and any agent $a\in\mathcal{A}$, let $\omega\sim_a\omega'$ if all edges along the unique path between nodes $\omega$ and $\omega'$ are labeled with agent $a$.
\end{definition}
\begin{lemma}\label{canonical sim is equivalence relation}
Relation $\sim_a$ is an equivalence relation on $\Omega$.
\qed
\end{lemma}

\begin{lemma}\label{transport lemma}
$\K_C\phi\in hd(\omega)$ iff $\K_C\phi\in hd(\omega')$,
if $\omega\sim_C \omega'$.
\end{lemma}

\begin{proof}
By Definition~\ref{canonical sim}, assumption $\omega\sim_C \omega'$ implies that all edges along the unique path between nodes $\omega$ and $\omega'$ are labeled with all agents of coalition $C$. Thus, it suffices to prove the statement of the lemma for any two adjacent vertices along this path. Let $\omega'=\omega::(D,b)::X$. Note that $C\subseteq D$ because edge $(\omega,\omega')$ is labeled with all agents in coalition $C$. We start by proving the first part of the lemma.

\noindent$(\Rightarrow)$ Suppose $\K_C\phi\in hd(\omega)$. Thus, $hd(\omega)\vdash \K_C\K_C\phi$ by Lemma~\ref{positive introspection lemma}. Hence, $hd(\omega)\vdash \K_D\K_C\phi$ by the Monotonicity axiom. Thus, $\K_D\K_C\phi\in hd(\omega)$ because set $hd(\omega)$ is maximal. Therefore, $\K_C\phi\in X=hd(\omega')$ by Definition~\ref{canonical outcome}.

\noindent$(\Leftarrow)$ Assume $\K_C\phi\notin hd(\omega)$. Thus, $\neg\K_C\phi\in hd(\omega)$ by the maximality of the set $hd(\omega)$. Hence, $hd(\omega)\vdash \K_C\neg\K_C\phi$ by the Negative Introspection axiom. Then, $hd(\omega)\vdash \K_D\neg\K_C\phi$ by the Monotonicity axiom. Thus, $\K_D\neg\K_C\phi\in hd(\omega)$ by the maximality of set $hd(\omega)$. Then, $\neg\K_C\phi\in X=hd(\omega')$ by Definition~\ref{canonical outcome}. Therefore, $\K_C\phi\notin hd(\omega')$ because set $hd(\omega')$ is consistent.
\end{proof}

\begin{corollary}\label{diamond corollary}
If $\omega\sim_C \omega'$, then
$\cK_C\phi\in hd(\omega)$ iff $\cK_C\phi\in hd(\omega')$.
\end{corollary}

The set of the initial states $I$ of the canonical game is the set of all equivalence classes of $\Omega$ with respect to relation $\sim_\mathcal{A}$.

\begin{definition}\label{canonical initial state}
$I=\Omega/\sim_\mathcal{A}$.
\end{definition}

\begin{lemma}\label{well-defined lemma}
Relation $\sim_C$ is well-defined on set $I$.
\end{lemma}
\begin{proof}
Consider outcomes $\omega_1,\omega_2,\omega'_1$, and $\omega'_2$ where $\omega_1\sim_C\omega_2$, $\omega_1\sim_\mathcal{A}\omega'_1$, and $\omega_2\sim_\mathcal{A}\omega'_2$. It suffices to show  $\omega'_1\sim_C\omega'_2$. Indeed, the assumptions $\omega_1\sim_\mathcal{A}\omega'_1$ and $\omega_2\sim_\mathcal{A}\omega'_2$ imply $\omega_1\sim_C\omega'_1$ and $\omega_2\sim_C\omega'_2$. Thus, $\omega'_1\sim_C\omega'_2$ because $\sim_C$ is an equivalence relation.
\end{proof}

\begin{corollary}\label{alpha iff omega}
$\alpha\sim_C\alpha'$ iff $\omega\sim_C\omega'$, for any states $\alpha,\alpha'\in I$, any outcomes $\omega\in\alpha$ and $\omega'\in\alpha'$, and any $C\subseteq\mathcal{A}$.
\end{corollary}

In~\cite{nt19aaai}, the domain of actions $\Delta$ of the canonical game is the set $\Phi$ of all formulae. Informally, if an agent employs action $\phi$, then she {\em vetoes} formula $\phi$. The set $P$ specifies under which conditions the veto takes place. Here, we modify this construction by requiring the agent, while vetoing formula $\phi$, to specify a coalition $C$ and an outcome $\omega$. The veto will take effect only if coalition $C$ cannot distinguish the outcome $\omega$ from the current outcome. One can think about this construction as requiring the veto ballot to be signed by a key only known, distributively, to coalition $C$. This way only coalition $C$ knows how the agent must vote. 

\begin{definition}
$\Delta =\{(\phi,C,\omega)\;|\;\phi\in\Phi, C\subseteq\mathcal{A}, \omega\in \Omega\}.$
\end{definition}

\begin{definition}\label{canonical play}
The set $P\subseteq I\times \Delta^\mathcal{A}\times \Omega$ consists of all triples $(\alpha,\delta,u)$ such that (i) $u\in\alpha$, and (ii)
    for any outcome $v$ and any formula $\cK_C\B^D_C\psi\in hd(v)$, if $\delta(a)=(\psi,C,v)$ for each agent $a\in D$ and $u\sim_C v$, then $\neg\psi\in hd(u)$.
\end{definition}

\begin{definition}\label{canonical pi}
$\pi(p)=\{(\alpha,\delta,\omega)\in P\;|\; p\in hd(\omega)\}$.
\end{definition}

This concludes the definition of the canonical game $G(X_0)$. In Lemma~\ref{termination lemma}, we show that this game satisfies the requirement of item (5) from Definition~\ref{game definition}. Namely, for each $\alpha\in I$ and each complete action profile $\delta\in\Delta^\mathcal{A}$, there is at least one $\omega\in \Omega$ such that $(\alpha,\delta,\omega)\in P$.

As usual, the completeness follows from the induction (or ``truth'') Lemma~\ref{induction lemma}. To prove this lemma we first need to establish a few auxiliary properties of game $G(X_0)$.

\begin{lemma}\label{B child exists lemma}
For any play $(\alpha,\delta,\omega)\in P$ of  game $G(X_0)$, any formula $\neg(\phi\to \B_C^D\phi)\in hd(\omega)$, and any profile $s\in\Delta^D$, there is a play $(\alpha',\delta',\omega')\in P$ such that $\alpha\sim_C\alpha'$, $s =_D\delta'$, and $\phi\in hd(\omega')$.
\end{lemma}
\begin{proof}
Let the complete action profile $\delta'$ be defined as:
\begin{equation}\label{choice of delta'}
    \delta'(a)=
    \begin{cases}
    s(a), & \mbox{ if } a\in D,\\
    (\bot,\varnothing,\omega), & \mbox{ otherwise}.
    \end{cases}
\end{equation}
Then, $s=_D\delta'$.
Consider the following set of formulae:
\begin{eqnarray*}
X&\!\!=\!\!&\!\{\phi\}\;\cup\;\{\psi\;|\;\K_C\psi\in hd(\omega)\}\\
&&\!\cup\;\{\neg\chi\;|\;\cK_E\B_E^F\chi\in hd(v), E\subseteq C, F\subseteq D, \\
&&\hspace{12mm} \forall a\in F(\delta'(a)=(\chi,E,v)), \omega\sim_E v\}.
\end{eqnarray*}

\begin{claim}
Set $X$ is consistent.
\end{claim}
\begin{proof-of-claim}
Suppose the opposite. Thus, there are formulae $\cK_{E_1}\B_{E_1}^{F_1}\chi_1,\dots,\cK_{E_n}\B_{E_n}^{F_n}\chi_n$, outcomes $v_1,\dots,v_n\in\Omega$, 
\begin{eqnarray}
\mbox{ and formulae }&&\hspace{-5mm}\K_C\psi_1,\dots,\K_C\psi_m\in hd(\omega),\label{choice of psi-s}\\
\mbox{such that }&&\hspace{-5mm}\cK_{E_i}\B_{E_i}^{F_i}\chi_i\in hd(v_i)\;\forall i\le n,\label{choice of chi2-s}\\
&&\hspace{-5mm}E_1,\dots,E_n\subseteq C,\hspace{2mm}F_1,\dots,F_n\subseteq D,\label{choice of Es}\\
&&\hspace{-5mm}\delta'(a)=(\chi_i,E_i,v_i)\;\forall i\le n\;\forall a\in F_i,\label{choice of votes}\\
&&\hspace{-5mm}\omega\sim_{E_i} v_i\;\forall i\le n,\label{choice of omega'2}\\
\mbox{ and }&&\hspace{-5mm}\psi_1,\dots,\psi_m,\neg\chi_1,\dots,\neg\chi_n\vdash\neg\phi.\label{choice of cons}
\end{eqnarray}
Without loss of generality, we assume that formulae $\chi_1,\dots,\chi_n$ are distinct. Thus, assumption~(\ref{choice of votes}) implies that $F_1,\dots,F_n$ are pairwise disjoint. 
Assumption~(\ref{choice of cons}) implies 
$$
\psi_1,\dots,\psi_m\vdash\phi\to\chi_1\vee\dots\vee\chi_n
$$
by the propositional reasoning.
Then, 
$$
\K_C\psi_1,\dots,\K_C\psi_m\vdash\K_C(\phi\to\chi_1\vee\dots\vee\chi_n)
$$
by Lemma~\ref{super distributivity}.
Hence, by assumption~(\ref{choice of psi-s}),
$$
hd(\omega)\vdash\K_C(\phi\to\chi_1\vee\dots\vee\chi_n).
$$

At the same time, $\cK_{E_1}\B^{F_1}_{E_1}\chi_1,\dots, \cK_{E_n}\B^{F_n}_{E_n}\chi_n\in hd(\omega)$ by assumption~(\ref{choice of chi2-s}), assumption~(\ref{choice of omega'2}), and Corollary~\ref{diamond corollary}.
Thus, $hd(\omega)\vdash \K_C(\phi\to\B^D_C\phi)$ by  Lemma~\ref{five plus plus}, assumption~(\ref{choice of Es}), and the assumption that sets $F_1,\dots,F_n$ are pairwise disjoint.
Hence, by the Truth axiom, $hd(\omega)\vdash \phi\to\B^D_C\phi$, which contradicts the assumption $\neg(\phi\to\B^D_C\phi)\in hd(\omega)$ of the lemma because set $hd(\omega)$ is consistent.
Thus, $X$ is consistent.
\end{proof-of-claim}

Let $X'$ be any maximal consistent extension of set $X$ and $\omega'_b$ be the sequence $\omega::(C, b)::X'$ for each element $b\in\mathcal{B}$. 
Then, $\omega'_b\in\Omega$ for each element $b\in\mathcal{B}$ by Definition~\ref{canonical outcome} and the choice of sets $X$ and $X'$. Also $\phi\in X\subseteq hd(\omega'_b)$ for each $b\in\mathcal{B}$ by the choice of sets $X$ and $X'$. 

Note that family $\{Tree(\omega'_b)\}_{b\in \mathcal{B}}$ consists of pair-wise disjoint sets. This family has the same cardinality as set $\mathcal{B}$. Let  
$$
V=\{v\in \Omega\;|\; \delta'(a)=(\psi,E,v), a\in \mathcal{A}, \psi\in\Phi,E\subseteq\mathcal{A}\}.
$$
The cardinality of $V$ is at most the cardinality of set $\mathcal{A}$. By the choice of set $\mathcal{B}$, its  cardinality is larger than the cardinality of set $\mathcal{A}$. Thus, there exists a set $Tree(\omega'_{b_0})$ in family $\{Tree(\omega'_b)\}_{b\in \mathcal{B}}$ disjoint with set $V$:
\begin{equation}\label{disjoint equation}
    Tree(\omega'_{b_0}) \cap V=\varnothing.
\end{equation}
Let $\omega'$ be the outcome $\omega'_{b_0}$.

\begin{figure}[ht]
\begin{center}
\vspace{-2mm}
\scalebox{0.5}{\includegraphics{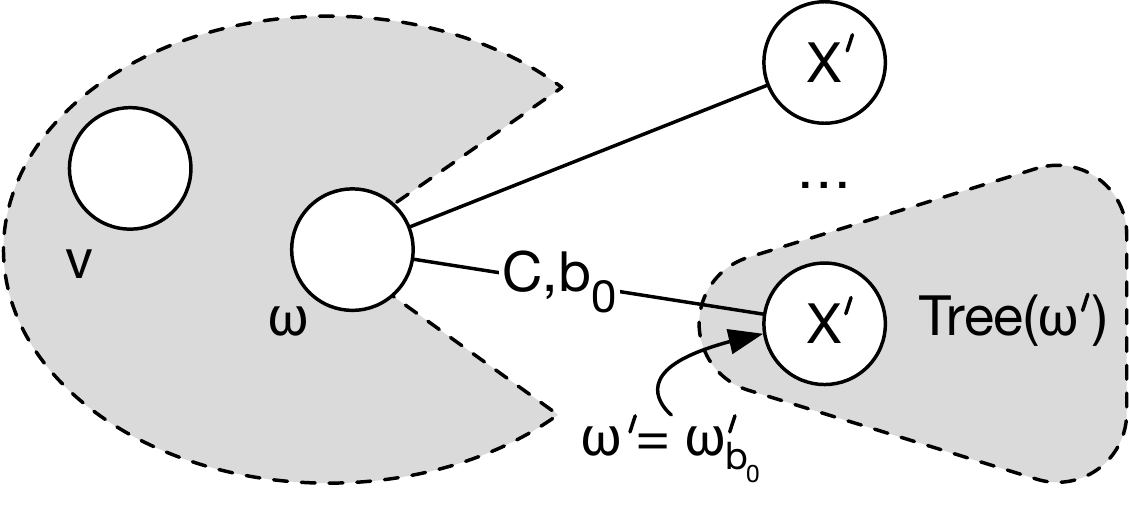}}
\caption{Towards the Proof of Claim~\ref{E subseteq C claim}.}\label{signature figure}
\vspace{-4mm}
\end{center}
\vspace{-2mm}
\end{figure}

\begin{claim}\label{E subseteq C claim}
If $\omega' \sim_E v$ for some $v\in V$, then $E\subseteq C$.
\end{claim}
\begin{proof-of-claim}
Consider any agent $a\in E$. By Definition~\ref{canonical sim}, assumption $\omega' \sim_E v$ implies that each edge along the unique path connecting vertex $\omega$ with vertex $v$ is labeled with agent $a$. At the same time, $v\notin Tree(\omega')$ by statement~(\ref{disjoint equation}) and because  $\omega'=\omega'_{b_0}$. Thus, the path between vertex $\omega'$ and vertex $v$  must go through vertex $\omega$, see Figure~\ref{signature figure}. Hence, this path must contain edge $(\omega',\omega)$. Since all edges along this path are labeled with agent $a$ and edge $(\omega',\omega)$ is labeled with agents from set $C$, it follows that $a\in C$.
\end{proof-of-claim}


Let initial state $\alpha'$ be the equivalence class of outcome $\omega'$ with respect to the equivalence relation $\sim_{\mathcal{A}}$. Note that $\omega\sim_C\omega'$ by Definition~\ref{canonical outcome} because $\omega'=\omega::(C,b_0)::X'$. Therefore, $\alpha\sim_C\alpha'$ by Corollary~\ref{alpha iff omega}.

\begin{claim}
$(\alpha',\delta',\omega')\in P$.
\end{claim}
\begin{proof-of-claim}
First, note that $\omega'\in\alpha'$ because initial state $\alpha'$ is the equivalence class of outcome $\omega'$. Next, consider an outcome $v\in\Omega$
\begin{eqnarray}
    \mbox{and a formula }&&\hspace{-5mm}\cK_E\B_E^F\chi\in hd(v), \label{K bar}\\
    \mbox{such that }&&\hspace{-5mm}\omega'\sim_E v,\hspace{3mm}\label{omega' E v}\\
    \mbox{and }&&\hspace{-5mm}\forall a\in F\, (\delta'(a)=(\chi,E,v)).\label{delta'}
    \end{eqnarray}
By Definition~\ref{canonical play}, it suffices to show that $\neg\chi\in hd(\omega')$. 

\vspace{1mm}
\noindent{\bf Case I:} $F=\varnothing$.
Then, $\neg\B^F_E\chi$ is an instance of the None to Act axiom. Thus, $\vdash\K_E\neg\B^F_E\chi$ by the Necessitation inference rule. Hence, $\neg\K_E\neg\B^F_E\chi\notin hd(v)$ by the consistency of the set $hd(v)$, which contradicts the assumption~(\ref{K bar}) and the definition of modality $\cK$. 

\vspace{1mm}
\noindent{\bf Case II:} $\varnothing\neq F\subseteq D$. 
Thus, there exists an agent $a\in F$. Note that $\delta'(a) = (\chi, E, v)$ by assumption~(\ref{delta'}). Hence, $v\in V$ by the definition of set $V$. Thus, $E\subseteq C$ by Claim~\ref{E subseteq C claim} and assumption~(\ref{omega' E v}).
Then, $\neg\chi\in X$ by the definition of set $X$, the assumption of the case that $F\subseteq D$, assumption~(\ref{K bar}), assumption~(\ref{omega' E v}), and  assumption (\ref{delta'}). Therefore, $\neg\chi\in hd(\omega')$ because $X\subseteq X'=hd(\omega'_{b_0})=hd(\omega')$ by the choice of set $X'$, set of sequences $\{\omega'_{b}\}_{b\in\mathcal{B}}$, and sequence $\omega'$.  

\vspace{1mm}
\noindent{\bf Case III:} $F\nsubseteq D$. Consider any agent $a\in F\setminus D$. Thus, $\delta'(a)=(\bot,\varnothing,\omega)$ by equation~(\ref{choice of delta'}). Thus, $\chi\equiv\bot$ by statement~(\ref{delta'}) and the assumption $a\in F$. Hence, formula $\neg\chi$ is a tautology. Therefore, $\neg\chi\in hd(\omega')$ by the maximality of set $hd(\omega')$. 
\end{proof-of-claim}
This concludes the proof of the lemma.
\end{proof}

\begin{lemma}\label{delta exists lemma}
For any outcome $\omega\in\Omega$, there is a state $\alpha\in I$ and a complete profile $\delta\in \Delta^\mathcal{A}$ such that $(\alpha,\delta,\omega)\in P$.
\end{lemma}
\begin{proof}
Let initial state $\alpha$ be the equivalence class of outcome $\omega$ with respect to the equivalence relation $\sim_{\mathcal{A}}$. Thus, $\omega\in\alpha$. Let $\delta$ be the complete profile such that $\delta(a)=(\bot,\varnothing,\omega)$ for each $a\in \mathcal{A}$. To prove $(\alpha,\delta,\omega)\in P$, consider any outcome $v\in\Omega$, any formula $\cK_C\B_C^D\chi\in hd(v)$ 
such that 
\begin{equation}\label{delta = chi}
    \forall a\in D\,(\delta(a)=(\chi, C,v)).
\end{equation}
By Definition~\ref{canonical play}, it suffices to show that $\neg\chi\in hd(\omega)$. 

\vspace{1mm}
\noindent{\bf Case I}: $D=\varnothing$. Thus, $\vdash\neg\B_C^D\chi$ by the None to Act axiom. Hence, $\vdash\K_C\neg\B_C^D\chi$ by the Necessitation rule. Then, $\neg\K_C\neg\B_C^D\chi\notin hd(v)$ because set $hd(v)$ is consistent. Therefore, $\cK_C\B_C^D\chi\notin hd(v)$ by the definition of modality $\cK$, which contradicts the choice of $\cK_C\B_C^D\chi$. 

\vspace{1mm}
\noindent{\bf Case II}: $D\neq\varnothing$. Then, there is an agent $a\in D$. Thus, $\delta(a)=(\chi, C, v)$ by statement~(\ref{delta = chi}).  Hence, $\chi\equiv\bot$ by the definition of action profile $\delta$. Then, $\neg\chi$ is a tautology. Therefore, $\neg\chi\in hd(\omega)$ by the maximality of set $hd(\omega)$.
\end{proof}

\begin{lemma}\label{termination lemma}
For each $\alpha\in I$ and each complete action profile $\delta\in\Delta^\mathcal{A}$, there is at least one outcome $\omega\in \Omega$ such that $(\alpha,\delta,\omega)\in P$.
\end{lemma}
\begin{proof}
By Definition~\ref{canonical initial state}, initial state $\alpha$ is an equivalence class. Since each equivalence class is not empty, there must exist an outcome $\omega_0\in \Omega$ such that $\omega_0\in \alpha$. By Lemma~\ref{delta exists lemma}, there is an initial state $\alpha_0\in I$ and a complete action profile $\delta_0\in \Delta^\mathcal{A}$ such that $(\alpha_0,\delta_0,\omega_0)\in P$. Then, $\omega_0\in \alpha_0$ by Definition~\ref{canonical play}. Hence, $\omega_0$ belongs to both equivalence classes $\alpha$ and $\alpha_0$. Thus, $\alpha=\alpha_0$. Therefore, $(\alpha,\delta_0,\omega_0)\in P$.
\end{proof}

\begin{lemma}\label{N child exists lemma}
For any play $(\alpha,\delta,\omega)\in P$ and any  $\neg\K_C\phi\in hd(\omega)$,  there is a play $(\alpha',\delta',\omega')\in P$ such that $\alpha\sim_C\alpha'$ and $\neg\phi\in hd(\omega')$.
\end{lemma}
\begin{proof}
Consider the set $X=\{\neg\phi\}\;\cup\;\{\psi\;|\;\K_C\psi\in hd(\omega)\}$. First, we show that set $X$ is consistent. Suppose the opposite. Then, there are formulae  $\K_C\psi_1,\dots,\K_C\psi_n\in hd(\omega)$
such that
$
\psi_1,\dots,\psi_n\vdash\phi.
$
Hence, 
$
\K_C\psi_1,\dots,\K_C\psi_n\vdash\K_C\phi
$
by Lemma~\ref{super distributivity}.
Thus, $hd(\omega)\vdash\K_C\phi$ because $\K_C\psi_1,\dots,\K_C\psi_n\in hd(\omega)$. Hence, $\neg\K_C\phi\notin hd(\omega)$ because set $hd(\omega)$ is consistent, which contradicts the assumption of the lemma. Therefore, set $X$ is consistent.

Recall that set $\mathcal{B}$ has larger cardinality than set $\mathcal{A}$. Thus, there is at least one $b\in\mathcal{B}$.
Let set $X'$ be any maximal consistent extension of set $X$ and  $\omega'$ be the sequence $\omega::(C,b)::X'$. Note that $\omega'\in\Omega$ by Definition~\ref{canonical outcome} and the choice of sets $X$ and $X'$.  Also, $\neg\phi\in X\subseteq X'=hd(\omega')$ by the choice of sets $X$ and $X'$. 

By Lemma~\ref{delta exists lemma}, there is an initial state $\alpha'\in I$ and a profile $\delta'\in \Delta^\mathcal{A}$ such that $(\alpha',\delta',\omega')\in P$.  Note that $\omega\sim_C\omega'$ by Definition~\ref{canonical sim} and the choice of $\omega'$. Thus, $\alpha\sim_C\alpha'$ by Corollary~\ref{alpha iff omega}.
\end{proof}

\begin{lemma}\label{induction lemma}
$(\alpha,\delta,\omega)\Vdash\phi$ iff $\phi\in hd(\omega)$.
\end{lemma}
\begin{proof}
We prove the lemma by induction on the complexity of formula $\phi$. If $\phi$ is a propositional variable, then the lemma follows from Definition~\ref{sat} and Definition~\ref{canonical pi}. If formula $\phi$ is an implication or a negation, then the required follows from the induction hypothesis and the maximality and the consistency of set $hd(\omega)$ by Definition~\ref{sat}.
%
Assume that formula $\phi$ has the form $\K_C\psi$.

\noindent $(\Rightarrow):$ Let $\K_C\psi\notin hd(\omega)$. Thus, $\neg\K_C\psi\in hd(\omega)$ by the maximality of set $hd(\omega)$. Hence, by Lemma~\ref{N child exists lemma}, there is a play $(\alpha',\delta',\omega')\in P$ such that $\alpha\sim_C\alpha'$ and $\neg\psi\in hd(\omega')$. Then, $\psi\notin hd(\omega')$ by the consistency of set $hd(\omega')$. Thus, $(\alpha',\delta',\omega')\nVdash\psi$ by the induction hypothesis. Therefore, $(\alpha,\delta,\omega)\nVdash\K_C\psi$ by Definition~\ref{sat}.

\vspace{.5mm}

\noindent $(\Leftarrow):$ Let $\K_C\psi\in hd(\omega)$. Thus, $\psi\in hd(\omega')$ for any $\omega'\in\Omega$ such that $\omega\sim_C\omega'$, by Lemma~\ref{transport lemma}. Hence, by the induction hypothesis, $(\alpha',\delta',\omega')\Vdash\psi$ for each play $(\alpha',\delta',\omega')\in P$ such that  $\omega\sim_C\omega'$. Thus, $(\alpha',\delta',\omega')\Vdash\psi$ for each  $(\alpha',\delta',\omega')\in P$ such that  $\alpha\sim_C\alpha'$, by Lemma~\ref{alpha iff omega}. Therefore, $(\alpha,\delta,\omega)\Vdash\K_C\psi$ by Definition~\ref{sat}. 

Assume formula $\phi$ has the form $\B^D_C\psi$. 

\noindent $(\Rightarrow):$ Suppose $\B^D_C\psi\notin hd(\omega)$. 

\noindent{\bf Case I}: $\psi\notin hd(\omega)$. Then, $(\alpha,\delta,\omega)\nVdash\psi$ by the induction hypothesis. Thus, $(\alpha,\delta,\omega)\nVdash\B^D_C\psi$ by Definition~\ref{sat}. 

\noindent{\bf Case II}: $\psi\in hd(\omega)$. Let us show that $\psi\to\B^D_C\psi\notin hd(\omega)$. Indeed, if $\psi\to\B^D_C\psi\in hd(\omega)$, then $hd(\omega)\vdash \B^D_C\psi$ by the Modus Ponens rule. Thus, $\B^D_C\psi\in hd(\omega)$ by the  maximality of set $hd(\omega)$, which contradicts the assumption above.

Since set $hd(\omega)$ is maximal, statement $\psi\to\B^D_C\psi\notin hd(\omega)$ implies that $\neg(\psi\to\B^D_C\psi)\in hd(\omega)$. Hence, by Lemma~\ref{B child exists lemma}, for any action profile $s\in \Delta^D$, there is a play $(\alpha',\delta',\omega')$ such that $\alpha\sim_C\alpha'$, $s=_D\delta'$, and $\psi\in hd(\omega')$. Thus, by the induction hypothesis, for any action profile $s\in \Delta^D$, there is a play $(\alpha',\delta',\omega')$ such that $\alpha\sim_C\alpha'$, $s=_D\delta'$, and $(\alpha',\delta',\omega')\Vdash \psi$. Therefore, $(\alpha,\delta,\omega)\nVdash\B^D_C\psi$ by Definition~\ref{sat}.

\vspace{1mm}
\noindent $(\Leftarrow):$ Let $\B^D_C\psi\in hd(\omega)$. Hence, $hd(\omega)\vdash\psi$ by the Truth axiom. Thus, $\psi\in hd(\omega)$ by the maximality of the set $hd(\omega)$. Then, $(\alpha,\delta,\omega)\Vdash\psi$ by the induction hypothesis.

Next, let $s\in \Delta^D$ be the action profile of coalition $D$ such that $s(a)=(\psi,C,\omega)$ for each agent $a\in D$. Consider any play $(\alpha',\delta',\omega')\in P$ such that $\alpha\sim_C\alpha'$ and $s=_D\delta'$. By Definition~\ref{sat}, it suffices to show that  $(\alpha',\delta',\omega')\nVdash \psi$. 

Assumption $\B^D_C\psi\in hd(\omega)$  implies $hd(\omega)\nvdash \neg\B^D_C\psi$ because set $hd(\omega)$ is consistent. Thus, $hd(\omega)\nvdash \K_C\neg\B^D_C\psi$ by the contraposition of the Truth axiom. Hence, $\neg\K_C\neg\B^D_C\psi\in hd(\omega)$ by the maximality of $hd(\omega)$.  Then, $\cK_C\B^D_C\psi\in hd(\omega)$ by the definition of modality $\cK$.
Recall that $s(a)=(\psi,C,\omega)$ for each agent $a\in D$ by the choice of the action profile $s$. Also, $s=_D\delta'$ by the choice of the play  $(\alpha',\delta',\omega')$. Hence, $\delta'(a)=(\psi,C,\omega)$ for each agent $a\in D$. Thus, $\neg\psi\in hd(\omega')$ by Definition~\ref{canonical play} and because $\cK_C\B^D_C\psi\in hd(\omega)$ and $(\alpha',\delta',\omega')\in P$. Then,  $\psi\notin hd(\omega')$ by the consistency of set $hd(\omega')$. Therefore, $(\alpha',\delta',\omega')\nVdash \psi$ by the induction hypothesis.
\end{proof}

Next is the strong completeness theorem for our system.
\begin{theorem}\label{completeness theorem}
If $X\nvdash\phi$, then there is a game, and a play $(\alpha,\delta,\omega)$ of this game such that $(\alpha,\delta,\omega)\Vdash\chi$ for each $\chi\in X$ and $(\alpha,\delta,\omega)\nVdash\phi$.
\end{theorem}
\begin{proof}
Assume that $X\nvdash\phi$. Hence, set $X\cup\{\neg\phi\}$ is consistent. Let $X_0$ be any maximal consistent extension of set $X\cup\{\neg\phi\}$ and let game $\left(I,\{\sim_a\}_{a\in\mathcal{A}},\Delta,\Omega,P,\pi\right)$ be the canonical game $G(X_0)$. Also, let $\omega_0$ be the single-element sequence $X_0$. Note that $\omega_0\in \Omega$ by Definition~\ref{canonical outcome}. 
By Lemma~\ref{delta exists lemma},  there is an initial state $\alpha\in I$ and a complete action profile $\delta\in \Delta^\mathcal{A}$ such that $(\alpha,\delta,\omega_0)\in P$. Hence, $(\alpha,\delta,\omega_0)\Vdash\chi$ for each $\chi\in X$ and $(\alpha,\delta,\omega_0)\Vdash\neg\phi$ by Lemma~\ref{induction lemma} and the choice of set $X_0$. Thus,  $(\alpha,\delta,\omega_0)\nVdash\phi$ by Definition~\ref{sat}.
\end{proof}

\section{Conclusion}\label{conclusion section}

In this paper, we proposed a formal definition of the second-order blameworthiness or duty to warn in the setting of strategic games. Our main technical result is a sound and complete logical system that describes the interplay between the second-order blameworthiness and the distributed knowledge modalities. 

\label{end of paper}

\bibliographystyle{ACM-Reference-Format}
\bibliography{sp}

\end{document}